\documentclass{dyno2025}
\setcitestyle{numbers}

\usepackage{times, soul}  
\usepackage{helvet,nicefrac} 
\usepackage{courier}  
\usepackage[small]{caption} 
\usepackage{subcaption}
\usepackage{mwe}
\usepackage{xcolor,amsmath,amsfonts,amssymb}

\usepackage{todonotes}
\usepackage{algorithm}
\usepackage{algorithmic}

\newcommand{\cV}{{\mathcal{U}}}

\newtheorem{assumption}{Assumption}

\makeatletter
\newcommand{\blfootnote}[1]{%
  \begingroup
  \renewcommand\thefootnote{}%
  \renewcommand\@makefnmark{}%
  \long\def\@makefntext##1{\noindent ##1}%
  \footnote{#1}%
  \addtocounter{footnote}{-1}%
  \endgroup
}
\makeatother

\makeatletter
\@ifundefined{proof}{%
  \newenvironment{proof}[1][Proof]{%
    \par\noindent\textbf{#1. }
  }{%
    \hfill$\square$\par\medskip
  }%
}{}
\makeatother

\title{Privacy Preserving Gossip Learning}

\optauthor{%
\Name{Erkan Bayram} \Email{ebayram2@illinois.edu}\\
\addr University of Illinois Urbana Champaign, IL, USA
\AND
\Name{Mohamed-Ali Belabbas} \Email{belabbas@illinois.edu}\\
\addr University of Illinois Urbana Champaign, IL, USA
\AND
\Name{Tamer Ba\c{s}ar} \Email{basar1@illinois.edu}\\
\addr University of Illinois Urbana Champaign, IL, USA}

\begin{document}

\maketitle
\blfootnote{Accepted to the 65th IEEE Conference on Decision and Control (CDC), 2026. Research was supported in part by ARO Grant W911NF-24-1-0085, ARO Grant W911NF-24-1-0105 and ARO Grant W911NF-26-1-A335}

\begin{abstract}%

{
We propose a decentralized privacy-preserving learning algorithm in which each agent holds a single private sample and a shared model. Samples are learned sequentially, and each update must preserve the endpoint mappings at previously learned samples while protecting private data. This gives each agent three roles: (i) a learner that updates the model parameters, (ii) a teacher whose sample is learned at the current iteration, and (iii) a protected agent whose sample has already been learned. We build on Tuning without Forgetting (TwF) method to preserve previously learned mappings and show that TwF provides an indistinguishability guarantee for the learner whenever the set of protected agents contains another sample with the same label. For the teacher, we formulate a minimax optimal control problem that models the differential privacy noise as a worst-case disturbance to prevent performance loss while maintaining the same level of privacy for the gradient. For the protected agents, we compute the projections locally and aggregate them using a private push-sum gossip protocol. We prove geometric convergence of the decentralized gossip algorithm and of the distributed projection for TwF.
}

\end{abstract}

\section{Introduction}\label{sec:introduction}

Privacy in distributed learning is typically enforced at the level of individual gradient updates. Differential privacy (DP) techniques~\cite{dwork2014algorithmic} have been proposed to mitigate gradient leakage attacks, through which private client-local training samples can be reconstructed from shared gradients~\cite{melis2019exploiting,huang2021evaluating,zhu2019deep,wei2021user,el2022differential}. However, as training progresses or new agents join the network, the shared model parameters accumulate information about previously learned samples. Consequently, per-update privacy guarantees alone may not adequately protect samples that have already been incorporated into the model. This creates a persistent privacy risk associated with the model parameters themselves~\cite{shokri2017membership,yang2023dynamic}.

To address this, we propose an iterative privacy-preserving training algorithm in which agents sequentially contribute to a shared model. Each agent holds a single private sample and communicates over a network. The role of each agent evolves over the course of training. An agent begins as the~\emph{teacher agent}, and~\emph{the learner} memorizes its sample under differential privacy. Once the teacher agent's sample has been learned, it becomes a~\emph{protected agent}, ensuring that its sample is not forgotten in subsequent model updates. We model the distributed supervised learning problem as an optimal control problem over an ensemble of points. The shared model corresponds to a controlled ODE whose open-loop control function is updated sequentially across agents. Based on the Neural ODE formulation~\cite{chen2018neural,dupont2019augmented,bayram2024control}, our method is compatible with deep residual neural networks through Euler discretization.

 In our framework, the teacher agent first generates a differentially private update to protect its sample. However, DP noise injection degrades model performance~\cite{park2023differentially}. To mitigate this,  we model the injected noise as a control 
disturbance and design an update robust to worst-case perturbations by solving a minimax optimal control problem~\cite{bayram2025control}. This approach provides differential privacy for the teacher agent while mitigating the performance degradation caused by DP noise. We validate the privacy–performance tradeoff numerically and compare against the non-robust approach of directly injecting DP noise under the same privacy budget.

Due to the iterative nature of the algorithm, each update to the shared control affects all previously learned mappings. To prevent forgetting of previously learned samples, we combine gossip consensus with Tuning without Forgetting (TwF), a recently proposed method for iterative training of Neural ODEs, in which training samples are added sequentially and the control update at each step is projected onto a prescribed set to maintain previously learned mappings~\cite{bayram2024control,bayram2025geometric}. In our method, protected agents compute their local projections in parallel, and aggregate the results through a privacy-preserving gossip protocol. It ensures that the mappings at previously learned samples are preserved while maintaining privacy guarantees for the protected agents.

Our main contributions in this paper are as follows:
\begin{itemize}

\item We prove that the decentralized projection algorithm converges at a
geometric rate to the projection onto the intersection of the tangent spaces of
the protected agents.
\item We establish an explicit geometric rate for the
private push-sum protocol of~\cite{gao2018privacy}.

\item We show that {\em Tuning without Forgetting} prevents adversaries from
uniquely identifying the learner's sample from the shared control sequence
whenever the protected set contains another sample with the same label.
    \item We provide theoretical privacy guarantees across all three agent roles (1) the teacher agent via an $(\varepsilon,\delta)$-DP robust control update, (2) the protected agents via a private push-sum gossip protocol, and (3) the learner agent via indistinguishability under the TwF invariance structure.
\end{itemize}

\section{Preliminaries}\label{sec:prelim}

Consider a paired dataset $\{(\bar{x}^i, y^i)\}_{i=1}^N$, where $\bar{x}^i \in \mathbb{R}^{\bar{n}}$ are pairwise distinct input points $\bar{x}^i\neq \bar{x}^j$ for $j \neq i$ and $y^i \in \mathbb{R}^{n_o}$ are their corresponding labels. In the distributed setting, the dataset is partitioned across a network $G = ({V}, {E})$, where the vertex set ${V} = \{v_1,\cdots,v_N\}$ represents the agents and the edge set ${E} = \{e_1,\cdots,e_{|{E}|}\}$ represents the communication links. Each agent $v_i \in {V}$ holds a single training pair $(\bar{x}^i,y^i)$. 

For any subset  of $C \subset \mathbb{N}$ that indexes the agents in $V$, let $G_C  =(V_C,E_C)$ be {\em the subgraph of $G$ induced by $C$} where $V_C =\{v_r \in V \mid r \in C \}$ and $E_c = \{ e \in E \mid \mbox{ both endpoints of } e \mbox{ in } V_C \}$. For each agent $r$, let ${N}_{r,C}$ be the set of neighbors in $G_C$.

\paragraph{Neural ODE Model:} We first embed the initial states into a higher-dimensional space via the uplift map
$E:\mathbb{R}^{\bar n}\to\mathbb{R}^{n}$, $E(\bar{x})=(\bar{x},0,\ldots,0)=x$ with $ n \ge \bar{n}$. For the rest of the paper, we consider the embedded pairs $({x}^i,y^i)$. Let $\cV := L^2([0,1],\mathbb{R}^p)$. 

All agents share the following controlled dynamical system:
\begin{equation}\label{eqn:control_system}
    \dot{x}(t) = f(x(t),u(t)), \qquad u \in \cV,
\end{equation}
where $f  :\mathbb{R}^{ n}\times\mathbb{R}^{p} \to \mathbb{R}^{ n} $ is a smooth vector field with uniformly  bounded Jacobians and smooth Jacobians. The flow of this system defines the map  
 \begin{align}\label{eqn:flow}
     \varphi_t(u,x): \cV \times \mathbb{R}^{{n}}  \times [0,1] \to \mathbb{R}^{{n}}
 \end{align}
which yields the trajectory $t \mapsto \varphi_t(u,x)$ initialized at $x$ with control $u$. We denote the flow at time $1$ by $\varphi_1(x,u)$. We project the endpoints of flow through a readout map $R:\mathbb{R}^{{n}} \to \mathbb{R}^{n_o}$ that is bounded, linear, $1$-Lipschitz, and of full row-rank Jacobian. Then, we call the composition $R(\varphi(u,x))$ {\em the end-point mapping at $x$}. 

A control $u$ is said to \emph{memorize} the sample $(x^i,y^i)$ if
\begin{equation}\label{eqn:control_set}
    u \in U^i := \{u\in\cV\mid R(\varphi(u,x^i)) = y^i\}.
\end{equation}
Then, we define the per-sample loss function $J_i(u)$ as:
\begin{align}\label{eqn:per_sample}
J_i(u) := \frac12 \|R(\varphi(u,x^i)) - y^i\|^2 .
\end{align}
By definition, $u\in U^i$ if, and only if, $J_i(u)=0$. In supervised learning, the objective is to find a single open-loop control $u \in \bigcap_{i=1}^N U^i$, that is, an {\em open-loop} control $u$ that drives all input points to their desired outputs simultaneously.

\paragraph{Geometry of the control set.} We study the geometry of the set $\bigcap_{i=1}^N U^i$ under three main assumptions. First, the existence of a control $u \in \bigcap_{i=1}^N U^i$ follows from a corollary of the Chow-Rashevsky theorem~\cite[Thm. 1]{agrachev2020control}. We define:
\begin{align}
\Delta_{ij} := \{[x_1^\top,\dots,x_S^\top]^\top \in \mathbb{R}^{n S} \mid x_i = x_j,\ i \neq j\}. 
\end{align}

\begin{assumption}~\cite[Lem.~2.2]{bayram2025geometric}\label{ass:bracket}
The $N$-fold control vector fields 
$$\mathcal{F}^N_X := \{[f^\top(x_1,u),\dots,f^\top(x_N,u)]^\top \mid u\in\cV\}$$ 
are bracket-generating on $\mathbb{R}^{\bar n N} \setminus ( \bigcup_{i\neq j} \Delta_{ij} )$.   
\end{assumption}

Next, we linearize system~\eqref{eqn:control_system} along the trajectory $\varphi_t(u,x^i)$. Then, we define:
$
A_i^u(t) := \partial_x f(\varphi_t(u,x^i),u(t)), 
$ and $
B_i^u(t) := \partial_u f(\varphi_t(u,x^i),u(t)).
$
It yields the following linear time-varying system (see~\cite[Lemma 1]{bayram2024control}):
\begin{equation}\label{eqn:defn_ltv}
    \dot z(t) = A_i^u(t) z(t) + B_i^u(t) \delta u(t),
\end{equation}
where $z(t):=\varphi_t(u+\delta u,x^i)-\varphi_t(u,x^i)$. Let $\Phi_{(u,x^i)}(1,t)$ be the state transition matrix of the LTV system~\eqref{eqn:defn_ltv}. We define the operator that maps a control variation $\delta u\in\cV$ to the endpoint variation of the trajectory:
\begin{align}\label{eqn:defn_l_operator}
\mathcal{L}_{(u,x^i)}(\delta u)
:=
\int_0^1
\Phi_{(u,x^i)}(1,\tau) B_i^u(\tau)\delta u(\tau)\,d\tau.
\end{align}
A control $u$ is said to be \emph{nonsingular} at $x^i$ if the operator $\mathcal{L}_{(u,x^i)}$ is surjective~\cite{sontag1992universal}. Then, we have:

\begin{assumption}
For a given set $\{(\bar{x}^i, y^i)\}_{i=1}^N$, any control $u\in \cap_{i=1}^N U^i$ is nonsingular at
   $x^i$ for all $i=1,\!\cdots\!,\!N$.  
\end{assumption}

Finally, we define the response for $x^i$ at $u$ as follows:
\begin{align}\label{eqn:defn_mui}
    m^u_i(t)\!:=\!R\!\left(\Phi_{(u,x^i)}(1,t)\, B^u_i(t)\right) \in L^2([0,1],\mathbb{R}^{n_o \times p}).
\end{align}
 Then, we have the following assumption, its prevalence is discussed in~\cite{bayram2025geometric}:

\begin{assumption}
For a given set $\{(\bar{x}^i, y^i)\}_{i=1}^N$, the responses $\{m_i^u(t)\}_{i=1}^N$ are non-collinear over $\mathbb{R}$ for all $u \in \bigcap_{i=1}^N U^i$.
\end{assumption}
In the next section, we present the main result under Assumptions (A.1), (A.2) and (A.3).

\section{Main Results}

We propose an iterative algorithm to find a control function such that $u \in \cap_{i=1}^N U^i$ without revealing the corresponding paired set. Let $u^k$ represent the control function at iteration $k$. We consider one agent as the \emph{learner agent} $\ell$, who is responsible for updating the shared control function $u^k$ at every iteration. The learner agent first memorizes its own sample $(x^\ell, y^\ell)$ by finding a control $u^k \in U^\ell$. It then sequentially memorizes the samples of the remaining agents, one at a time. We call the agent whose sample is currently being learned the \emph{teacher agent} $q$. Importantly, in this sequential method, each update must preserve the endpoint mapping at samples that have already been memorized by the learner in order to {prevent forgetting} of these samples. We refer to the agents associated with these samples as \emph{protected agents} and denote the set of protected agents at iteration $k$ by $\mathcal{C}_\ell^k$.  More precisely, at each iteration $k$, the learner agent seeks a control update $\delta u^k = u^{k+1} - u^k$ that satisfies:
\begin{itemize}
    \item {Memorizing the sample of the teacher agent $q$:}
\begin{align}\label{eqn:condition}
    {J}_{q}(u^{k} + \delta u^k ) \leq  {J}_{q}(u^{k}).
\end{align}
\item {Preserving the endpoint mapping for protected agents:}
\begin{align}\label{eqn:invariant}
          R(\varphi(u^k+\delta u^k,x^r))= y^r , \quad \forall r \in \mathcal{C}_\ell^k .
\end{align}
\end{itemize}

To achieve these objectives, we propose an algorithm called \emph{Privacy-Preserving Gossip Learning} in Algorithm~\ref{alg:main}, which consists of three nested iterations.

\begin{algorithm}
\caption{Privacy-Preserving Gossip Learning}\label{alg:main}
\begin{algorithmic}[1]
\STATE \textbf{Initialize} $u^0, \mathcal{C}_\ell   \gets \mathrm{Algorithm~\ref{alg:phase1}}( x^\ell)$
\FOR{$q \in V \setminus \{\ell\}$}
    \REPEAT
    \STATE $\beta^0 \gets \textsc{RobustDP}(u, x^{q})$\hfill\textit{(Algorithm~\ref{alg:robust_dp})}

    \FOR{$s = 0$ \TO $S$}
        \STATE \textbf{for all} agent $r \in \mathcal{C}_\ell$ \textbf{(parallel)}

        \STATE \quad $\theta^s_r \gets  \textsc{LocalProjection}(\beta^s)$  \hfill \textit{(Algorithm~\ref{alg:project})}
        \STATE $\beta^{s+1} \gets  \textsc{PrivatePushSum}(\{\theta^s_r\}_{r \in \mathcal{C}_\ell})$  \hfill \textit{(Algorithm~\ref{alg:priv_gossip})}
    \ENDFOR
    \STATE $u \gets u - \alpha \beta^{S}$  \hfill \textit{Control update} 
    \UNTIL convergence for $ u \in U^q$
    \STATE $\mathcal{C}_\ell \gets 
    \mathcal{C}_\ell \cup \{q\}$ \hfill \textit{Expand the protected set} 
\ENDFOR
\end{algorithmic}
\end{algorithm}

In the proposed method, the learner first calls a subprocedure, i.e. Algorithm~\ref{alg:phase1} (see Line 1 in Algorithm~\ref{alg:main}), to establish privacy in the sense of indistinguishability. Since this phase is independent of the gossip learning procedure, its details are provided later in Section~\ref{subsec:privacy}.

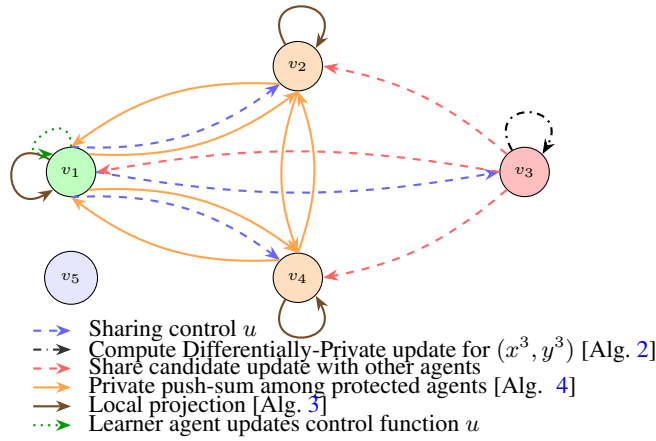
\begin{figure}
\centering
\usetikzlibrary{positioning, arrows.meta, shapes.geometric, fit, backgrounds}

\centering
\begin{tikzpicture}[
    learner/.style={circle, draw=black, fill=green!25,
                    minimum size=0.5cm, font=\tiny, align=center},
    target/.style={circle, draw=black, fill=red!25,
                    minimum size=0.5cm, font=\tiny, align=center},
    protected/.style={circle, draw=black, fill=orange!25,
                    minimum size=0.5cm, font=\tiny, align=center},
    regular/.style={circle, draw=black, fill=blue!10,
                    minimum size=0.7cm, font=\tiny, align=center},
    paneltitle/.style={font=\footnotesize\bfseries, align=center},
    steparrow/.style={-{Stealth[length=3mm]}, thick, gray!60}
]


\node[learner]   (v1) at (0, 1.4) {$v_1$};
\node[protected] (v2) at (3.0, 2.8) {$v_2$};
\node[target]    (v3) at (6.0, 1.4) {$v_3$};
\node[protected] (v4) at (3.0, 0) {$v_4$};
\node[regular]   (v5) at (0, 0) {$v_5$};

\draw[-{Stealth[length=2mm]}, thick, blue!60, dashed]
    (v1.east) to[bend right=10] (v3.west);

\draw[-{Stealth[length=2mm]}, thick, blue!60, dashed]
    (v1.north) to[bend right=20] (v2.south west);

\draw[-{Stealth[length=2mm]}, thick, blue!60, dashed]
    (v1.south) to[bend left=20] (v4.north west);

\draw[-{Stealth[length=2mm]}, thick, red!60, dashed]
    (v3.west) to[bend right=10] (v1.east);

\draw[-{Stealth[length=2mm]}, thick, red!60, dashed]
    (v3.north west) to[bend right=15] (v2.east);

\draw[-{Stealth[length=2mm]}, thick, red!60, dashed]
    (v3.south west) to[bend left=15] (v4.east);

\draw[-{Stealth[length=2mm]}, thick, orange!70]
    (v1.north east) to[bend right=20] (v2.south);
\draw[-{Stealth[length=2mm]}, thick, orange!70]
    (v2.south west) to[bend right=20] (v1.north);

\draw[-{Stealth[length=2mm]}, thick, orange!70]
    (v4.north) to[bend right=20] (v2.south);
\draw[-{Stealth[length=2mm]}, thick, orange!70]
    (v2.south) to[bend right=20] (v4.north);

\draw[-{Stealth[length=2mm]}, thick, orange!70]
    (v1.south east) to[bend left=20] (v4.north);
\draw[-{Stealth[length=2mm]}, thick, orange!70]
    (v4.north west) to[bend left=20] (v1.south);

\draw[-{Stealth[length=2mm]}, thick, brown!60!black, solid]
    (v2) to[out=120, in=60, looseness=5] (v2.north east);

\draw[-{Stealth[length=2mm]}, thick, brown!60!black, solid]
    (v4) to[out=240, in=300, looseness=5] (v4.south east);

\draw[-{Stealth[length=2mm]}, thick, brown!60!black, solid]
    (v1) to[out=150, in=210, looseness=5] (v1.south west);

\draw[-{Stealth[length=2mm]}, thick, green!60!black, dotted]
    (v1) to[out=90, in=180, looseness=6] (v1.north west);

\draw[-{Stealth[length=2mm]}, thick, black, dash dot]
    (v3) to[out=120, in=60, looseness=5] (v3.north east);

\begin{scope}[yshift=-1.6cm, xshift=0.0cm]




    \draw[-{Stealth[length=2mm]}, blue!60, dashed, thick]
        (-0.5, 0.9) -- (0, 0.9);
    \node[font=\footnotesize, right] at (0.1, 0.9) {Sharing control $u$};

    \draw[-{Stealth[length=2mm]}, black!80, dash dot, thick]
        (-0.5, 0.65) -- (0, 0.65);
    \node[font=\footnotesize, right] at (0.1, 0.65) {Compute Differentially-Private update for $(x^3,y^3)$ [Alg.~\ref{alg:robust_dp}]};

    \draw[-{Stealth[length=2mm]}, red!60, dashed, thick]
        (-0.5, 0.4) -- (0, 0.4);
    \node[font=\footnotesize, right] at (0.1, 0.4) {Share candidate update with other agents};

    \draw[-{Stealth[length=2mm]}, orange!70, thick]
        (-0.5, 0.15) -- (0, 0.15);
    \node[font=\footnotesize, right] at (0.1, 0.15) {Private push-sum among protected agents [Alg. ~\ref{alg:priv_gossip}]};

    \draw[-{Stealth[length=2mm]}, brown!60!black, solid, thick]
        (-0.5, -0.1) -- (0, -0.1);
    \node[font=\footnotesize, right] at (0.1, -0.1) {Local projection [Alg.~\ref{alg:project}]};

    \draw[-{Stealth[length=2mm]}, green!60!black, dotted, thick]
        (-0.5, -0.35) -- (0, -0.35);
    \node[font=\footnotesize, right] at (0.1, -0.35) {Learner agent updates control function $u$};
\end{scope}

\end{tikzpicture}
\caption{Network illustration of one outer iteration of Algorithm~\ref{alg:main}. Nodes represent agents and self-loops denote local computation. Edges between nodes denote communication links.}\label{fig:graph_example}
\end{figure}

\paragraph{Outer Loop (Agent Selection).}  At each iteration, the learner selects a teacher $q \in V \setminus \{\ell\}$, whose sample will be memorized. Once the learner successfully memorizes the sample of the teacher $q$, i.e. when $u \in U^q$, that agent is then added to the protected set. The learner then proceeds to the next teacher and repeats the process until all samples in the paired set across the network have been memorized. 

\paragraph{Middle Loop (Iterative Projection).} Given a teacher agent $q$, it first {\em locally} computes a candidate control update by $\textsc{RobustDP}$ (see Algorithm~\ref{alg:robust_dp}), which returns a differentially private gradient for the solution of a minimax optimal control problem, which will be detailed in Section~\ref{subsec:learning_new}. This candidate update $\beta^0$ is shared with the protected agents. Each protected agent $r \in \mathcal{C}_\ell$ locally projects the update $\beta^0$ onto the tangent space of its local invariance set of controls, i.e. $U^r$~\eqref{eqn:control_set} via Algorithm~\ref{alg:project}, which returns $\theta^s_r$. This projection will be detailed in Section~\ref{subsec:projection}. It ensures that the endpoint mapping {\em only} at individual sample $x^r$ remains unchanged.

\paragraph{Inner Loop (Gossip Consensus).}
The protected agents start to collaborate without revealing their samples $(x^r,y^r)$ and their local projections $\theta_r^s$ to find the control update that ensures~\eqref{eqn:invariant} for all $r \in \mathcal{C}_\ell$. At each iteration $s$, the protected agents run a private push-sum gossip protocol via Algorithm~\ref{alg:priv_gossip} to reach a consensus over the uniform average of their locally projected updates $\{\theta^s_r\}_{r \in\mathcal{C}^k_\ell}$. Once they reach a consensus, this agreed value $\beta^{s+1}$ is returned to the middle loop. Then, each protected agent $r$ {\em again} projects $\beta^{s+1}$ and gets the corresponding $\theta_r^{s+1}$. This alternating procedure of gossip averaging and local projection is repeated until both converge. Both iterations are shown to converge at a geometric rate in Section~\ref{subsec:projection}. Finally, the learner updates the control $u^k$ by $\beta^s$, i.e. selects $\delta u^k=\beta^S$, which satisfies both~\eqref{eqn:condition} and~\eqref{eqn:invariant}. This control update procedure is repeated until $u^k \in U^q$. The implementation of Algorithm~\ref{alg:main} is illustrated in Fig.~\ref{fig:graph_example}. 

Since the roles of the agents evolve over the iterations, we provide the privacy guarantee for the teacher in Section~\ref{subsec:learning_new}, and those for the protected and learner agents in Section~\ref{subsec:privacy}.

\subsection{Privately Learning Pair of Teacher Agent}\label{subsec:learning_new} In this section, we provide privacy guarantees for the teacher agent using differential privacy (DP), and recover the performance loss introduced by DP noise through a robust optimal control formulation. 

Directly releasing the gradient $\nabla J_q(u^k)$ to satisfy~\eqref{eqn:condition} may reveal the private pair $(x^{q},y^{q})$~\cite{zhu2019deep}. To prevent this, we employ differential privacy (DP) inspired by~\cite{dwork2014algorithmic}. Intuitively, a randomized mechanism satisfies differential privacy if its output distribution changes {\em negligibly} when a single private sample is added or removed from the dataset. 

Two datasets are said to be \emph{adjacent} if they differ in exactly one sample. {Since each agent holds a single sample in our framework, we say that $D=(x^q,y^q)$ and $D'=(x^{q'},y^{q'})$ are \emph{adjacent}, written $D\sim D'$, whenever both
are single-sample datasets. Under this adjacency relation, we define \textbf{local} differential privacy as follows~\cite{kasiviswanathan2011can}:

\begin{definition}\label{defn:dp}[Differential Privacy]
A randomized function $\mathcal M: \mathcal{D} \to A$ satisfies $(\varepsilon,\delta)$-differential privacy on the subspace $A$ if for all adjacent
datasets $D \subset \mathcal{D}$ and $D' \subset \mathcal{D}$ and all measurable sets $ \mathcal{S} \subseteq A $, it holds that
\begin{align}
\Pr[\mathcal M(D)\in  \mathcal{S} ]
\le
e^\varepsilon
\Pr[\mathcal M(D')\in  \mathcal{S} ]
+\delta .
\end{align}
\end{definition}

\paragraph{Robustness against DP noise:}
DP can be achieved by injecting Gaussian noise into the gradient~\cite{dwork2014algorithmic}. However, noise injection perturbs the descent direction and degrades model 
performance~\cite{park2023differentially}. We address this through a robust optimal control formulation inspired by~\cite{bayram2025control}. We model the injected noise as a worst-case control disturbance and design an update direction that is robust to such perturbations. This leads to the following minimax objective:
\begin{align}\label{eqn:problem}
    \min_{ u \in \cV } \left(  \max_{\gamma \in \mathcal{U}_\rho } \left({J}_{q}(u + \gamma) - \lambda \| \gamma \| \right) \right)
\end{align}
where $\mathcal{U}_\rho = \{\gamma \in \cV \mid \|\gamma\| \le \rho\}$ is the disturbance set and $\lambda > 0$ is a regularization parameter. 
From~\cite[Lem.~2]{bayram2025control}, if a control $u \in U^q$, it minimizes both~\eqref{eqn:per_sample} and~\eqref{eqn:problem}, so both objectives serve to memorize the sample of agent $q$.

To solve~\eqref{eqn:problem}, we follow the steps in~\cite{bayram2025control}.  We first solve the inner maximization problem. To this end, we define the bounded linear operator:
\begin{align}\label{eqn:operator_K}
    \mathcal{K}^u_{q}: \cV \to \cV, 
    \quad 
    \gamma \mapsto \int_0^1 m^u_{q}(t)^\top m^u_{q}(\tau) \, 
    \gamma(\tau)\, d\tau.
\end{align}
From~\cite[Prop.~1]{bayram2025control}, the closed-form solution of the inner maximization in~\eqref{eqn:problem} for sufficiently large $\lambda > 0$ is
\begin{align}\label{eqn:epsilon_*}
    \gamma_{q,u}^*(t) = \rho \frac{ 
    (\mathcal{K}^u_{q} - \lambda I)^{-1} 
    m^u_{q}(t)^\top r^u_{q} }{ 
    \| (\mathcal{K}^u_{q} - \lambda I)^{-1} 
    m^u_{q}(t)^\top r^u_{q} \|_2 }
\end{align}
where $r_{q}^u = R(\varphi(u,x^{q})) - y^{q}$ is the residual at endpoint. Then, we plug $\gamma_{q,u}^*(t)$ into~\eqref{eqn:problem}, which yields the minimization problem: $
 \min_{u \in \cV} {J}_{q}\bigl(u + \gamma_{q,u}^* \bigr)
$, which can be solved iteratively. The corresponding gradient at perturbed control $\bar{u} := u + \gamma_{q,u}^*$ can be found as 
$$ \nabla J_q(\bar{u}) = m_q^{\bar{u}}(t)^\top\, r^{\bar{u}}_q$$
where $m_q^{\bar{u}}(t)$ is given in~\eqref{eqn:defn_mui}~\cite{bayram2025geometric}. Then, we clip 
$\nabla J_q(\bar{u})$ to bound its norm by $B>0$:
\begin{align}\label{eqn:gradient}
    {g}_{q,\bar{u}}(t): =  \dfrac{ \nabla J_q\bigl(u+\gamma^*_{q,u})}{\max(1, \|\nabla J_q\bigl(u+\gamma^*_{q,u})\|/B)}
\end{align}

\paragraph{Noise process in function space:}
Although ${g}_{q,\bar{u}}(t)$ lies in $\cV$, white noise does not 
lie in $\cV$ and therefore cannot be added directly to the 
clipped gradient. Hence, we use a finite Karhunen--Lo\`eve 
expansion to construct a noise process. Let $\Lambda$ be a 
finite $d$-dimensional subspace of $\cV$ with orthonormal 
basis $\{\zeta_i\}_{i=1}^d$. Let $P_\Lambda:\cV \to \Lambda$ 
be the orthogonal projection onto $\Lambda$. We define a 
Gaussian noise process on $\Lambda$ by
\begin{align}\label{eqn:defn_noise_process}
    \mathcal{W}_{\Lambda}(t) 
    \;:=\; \sum_{l=1}^{d} \chi_l\, \zeta_l(t),
    \qquad 
    \chi_l \overset{\mathrm{iid}}{\sim} 
    \mathcal{N}(0,\,\sigma^2 B^2),
\end{align}
where $\mathcal{N}(0,\sigma^2 B^2)$ denotes a Gaussian random 
variable with variance $\sigma^2 B^2$. We then have the 
following privacy guarantee for teacher $q$ on a given 
subspace $\Lambda$ at given iteration $k$.
\begin{proposition}\label{prop:dp_teacher}
Let $g$ be the clipped gradient associated with sample 
$(x^q,y^q)$ in~\eqref{eqn:gradient}, and let $\Lambda$ be 
a finite $d$-dimensional subspace of $\cV$. Consider the 
mechanism
\begin{align}\label{eqn:mechanism}
    \mathcal{M}_{\Lambda}(x^q,y^q) 
    \;:=\; P_{\Lambda}(g) + \mathcal{W}_\Lambda,
\end{align}
where $\mathcal{W}_\Lambda$ is given 
in~\eqref{eqn:defn_noise_process}. For any $\varepsilon >0$ 
and $\delta \in (0,1)$, if
\begin{align}\label{eqn:prop_bound}
    \sigma \;\ge\; \frac{\sqrt{8\log(1.25/\delta)}}{\varepsilon},
\end{align}
then $\mathcal{M}_\Lambda$ satisfies local $(\varepsilon,\delta)$-DP 
for the teacher agent $q$ on $\Lambda$.
\end{proposition}
 See Section~\ref{sec:proof} for a proof of 
Proposition~\ref{prop:dp_teacher}. The given mechanism in~\eqref{eqn:mechanism} is applied once per control update while $q$ is the teacher, and the sample $(x^q,y^q)$ is used at no other iteration. Let $K_q$ denote the number of control updates in the teacher phase of agent $q$, that is, the number
of times Line~10 of Algorithm~\ref{alg:main} is executed before $u\in\mathcal{U}^q$
holds. The total privacy budget scales linearly with the number of control updates; that is, agent $q$ satisfies $(K_q\varepsilon,K_q\delta)$-DP~\cite[Thm.~3.16]{dwork2014algorithmic}.

The projected gradient $P_\Lambda({g}_{q,\bar{u}})$ is protected 
in the sense of Definition~\ref{defn:dp} on the subspace $\Lambda$, while 
the gradient components outside $\Lambda$ are released without 
protection. To minimize the unprotected components and implement the 
mechanism in~\eqref{eqn:mechanism}, we introduce a time 
discretization. For a fixed number of discretization steps 
$T \in \mathbb{N}$ with $\Delta t = 1/T$, grid points 
$t_j = j/T$ for $j = 0,\ldots,T-1$, we define the 
piecewise-constant functions:
\begin{align}\label{eqn:time_basis}
    e_l(t)\!:=\!\sqrt{T}\hat{e}_b\mathbf{1}_{[t_j,\,t_{j+1})}(t),\mbox{ } l\!=\!pj+b,\mbox{ } b\!=\!1,\!\ldots\!,p
\end{align}
where $\hat{e}_b \in \mathbb{R}^p$ denotes the $b$-th 
standard basis vector and $\mathbf{1}_{[t_j,\,t_{j+1})}(t)$ is the indicator function of the interval $[t_j, t_{j+1})$.

Next, we define the subspace of $\cV$ spanned by the basis of discretization~\eqref{eqn:time_basis}, i.e. $\{e_l\}_{l=1}^{d}$ with $d := pT$ and we denoted it by $\Lambda_T$. The teacher agent releases the update at iteration $k$ via the mechanism $\mathcal{M}_{\Lambda_T}$ in~\eqref{eqn:mechanism} on the subspace $\Lambda_T$ as follows: 
\begin{align}\label{eqn:dp_gradient_update}
    \xi(t) \;:=\; P_{{\Lambda}_T}({g}_{q,\bar{u}}(t)) 
    + \mathcal{W}_{\Lambda_{T}}(t).
\end{align}
By Proposition~\ref{prop:dp_teacher}, the update $\xi$~\eqref{eqn:dp_gradient_update} satisfies \textbf{$(\varepsilon,\delta)$-DP for the teacher agent $q$ on the subspace $\Lambda_T$}, provided that~\eqref{eqn:prop_bound} holds. Since $\xi \in \Lambda_T$, its coefficient vector $\bar{\xi}\in\mathbb{R}^d$ with respect to the basis $\{e_l\}_{l=1}^{d}$ is well-defined and satisfies
$$\xi(t) = \sum_{l=1}^{d} \bar{\xi}_l\, e_l(t).$$ 

In Algorithm~\ref{alg:robust_dp}, we compute the coefficient vector $\bar{\xi} \in \mathbb{R}^d$ as follows. First, $m^u_q$ is approximated using $L_q$ by~\cite[Algorithm~1]{bayram2024control}. Then, the disturbance $\gamma_{q,u}^*$ is obtained via $\mathcal{K}^u_q \approx L_q^\top L_q$ from~\cite[Algorithm~1]{bayram2025control}. Finally, the released update $\bar{\xi}$ is returned. Therefore, by abuse of notation, we write $\xi$ for both the continuous released update~\eqref{eqn:dp_gradient_update} and the output of Algorithm~\ref{alg:robust_dp} for the rest of the paper.

From this construction, the norm of the unprotected component, i.e., the component of 
$g_{q,\bar{u}}(t)$ outside $\Lambda_T$, equals the interpolation error due to discretization, which decreases as $T$ increases.

\begin{algorithm}
\caption{\textsc{RobustDP-Computation} }\label{alg:robust_dp}
\begin{algorithmic}[1]
\STATE \textbf{Input:} $u \in \mathbb{R}^{pT}$, $(x^{q},y^{q})$
\STATE $L_{q} \gets 
R\big(\mathcal{L}_{(u, x^{q})}\big)$
\hfill \textit{see Algorithm~1 in~\cite{bayram2024control}}
\STATE $\gamma \gets 
\left( L_{q}^\top L_{q} - \lambda I \right)^{-1}
\nabla {J}_{q}(u)$
\STATE $\xi \gets
\operatorname{clip}(
\nabla_u {J}_{q}\big(u+ \rho\frac{\gamma}{\|\gamma\|}\big),
B)\!
+\!\mathcal{N}(0,\!\sigma^2\!B^2\!\Delta tI_{pT})$
\STATE \textbf{Return:} $\xi$
\end{algorithmic}
\end{algorithm}

\subsection{Preservation of end-point mapping (Not Forgetting)}\label{subsec:projection}

We consider the set of control functions that satisfy the invariance equation in~\eqref{eqn:invariant}, i.e. the set $ \cap_{r \in \mathcal{C}_\ell}  U^r$. Tuning without Forgetting (TwF) provides a way to select updates $\delta u^k$ such that $u^{k}+\delta u^k$  lies in that set~\cite{bayram2024control}. For sake of notational simplicity, we suppress the iteration index $k$ from $u^k$ and $\mathcal{C}_\ell^k$ when it is clear from the context. 

Under (A.1-A.3), \cite[Theorem~1]{bayram2025geometric} states that this set forms a Banach submanifold of $\cV$ of finite codimension. Thus, TwF  relies on this geometric structure and selects the update $\delta u$ that lies in the tangent space of this submanifold,
\begin{align}\label{eqn:tangent_space}
\delta u \in T_u(\cap_{r \in \mathcal{C}_\ell} U^r),
\end{align}
then the updated control $u+\delta u$ remains on the manifold $\bigcap_{r \in \mathcal{C}_\ell} U^r$. Hence, it preserves the endpoint mappings at all previously learned samples.

If the control update $\delta u^{k}$ is selected in~\eqref{eqn:tangent_space}, then the control $u^{k+1}$ is said to be \textbf{tuning for $\mathcal{C}_\ell^k \cup \{q\}$ without forgetting $\mathcal{C}_\ell^k$}~\cite[Definition 1]{bayram2024control}. To explicitly characterize the tangent space in~\eqref{eqn:tangent_space}, we define: 

\begin{align}\label{eqn:kernel_intersection}
\ker(u,\mathcal{C}_\ell):=\left\{ \delta u \in \mathcal U \mid R(\mathcal{L}_{(u,x^r)}(\delta u))\!=\!0,
\ \forall r\!\in\!\mathcal{C}_\ell
\right\}.
\end{align}
This set consists of all variations $\delta u$ whose corresponding variational equation~\eqref{eqn:defn_ltv} yields zero endpoint variation at initial point $x^r$ for all $ r \in \mathcal{C}_\ell$.
From~\cite[Corollary~1]{bayram2025geometric}, this space admits the following characterization:
\begin{align}\label{eqn:ker_tangent}
\ker(u,\mathcal{C}_\ell)
=
T_u ( \cap_{r \in \mathcal{C}_\ell} U^r ).
\end{align}
Hence, we propose to select the update $\delta u$ by projecting the released update 
$\xi$~\eqref{eqn:dp_gradient_update} onto $\ker(u,\mathcal{C}_\ell)$ 
to satisfy~\eqref{eqn:tangent_space}:
\begin{align}\label{eqn:defn_projector}
\delta u
:=
P^{u}_{\mathcal{C}_\ell}(\xi),
\text{ where }
P^{u}_{\mathcal{C}_\ell}(\xi)
:= \operatorname*{argmin}_{w \in \ker(u,\mathcal{C}_\ell)}
\|\xi - w\|_2 . 
\end{align}

However, this projection operator  requires access to all samples of the protected agents in $\mathcal{C}_\ell$, which violates their privacy. To address this, we first define the {\em local} projection operator $P^{u}_{r} : \cV \to T_u U^r$ associated with each protected agent $r \in \mathcal{C}_\ell$ for a given update $\beta \in \cV$:
\begin{align}
 P^{u}_{r}(\beta)
:= \operatorname{argmin}_{w \in \operatorname{ker}(u,\{r\})}
\|\beta - w\|_2 
\end{align}

If the control $u$ is nonsingular at $x^r$ (see A.2), this operator admits a
closed-form solution~\cite[Prop. 1]{bayram2025geometric}. Algorithm~\ref{alg:project}
computes $P^u_r(\cdot)$ for a protected agent $r$. It characterizes the local tangent
space $T_u U^r$ as the kernel of the operator $R(\mathcal{L}_{(u,x^r)}(\cdot))$ and approximates
this kernel by the right null space of the matrix $L_r \in \mathbb{R}^{n_o \times pT}$.


 \begin{algorithm}[h!]
\caption{\textsc{Local Projection} }\label{alg:project}
\begin{algorithmic}[1]
\STATE \textbf{Input:} $u \in \mathbb{R}^{pT}$, $(x^{r},y^{r}) , \beta^s \in \mathbb{R}^{pT}$
 \STATE $L_r \gets R\big(\mathcal{L}_{(u, x^{r})}\big)$ \hfill \textit{see Algorithm~1 in~\cite{bayram2024control}} 
\STATE $ \theta  \gets\!(I-L_r^\dagger L_r) \beta^s$ \hfill \mbox{Computation of }$P^{u}_{r}(\beta^s)$ 
\STATE \textbf{Return:} $\theta$
\end{algorithmic}
\end{algorithm}

Using these local operators, we propose a decentralized method to compute the projection onto the intersection $\ker(u,\mathcal{C}_\ell)(=\cap_{r \in \mathcal{C}_\ell} \ker(u,\{r\}))$ without revealing samples to other protected agents.

\subsubsection{Average Consensus For Projected Updates}\label{sss:gossip}

For a given middle loop iterate $\beta^s$, we reformulate the uniform averaging of $\{P_r^u(\beta^s)\}_{r\in\mathcal{C}_\ell}$ as a consensus problem over a subgraph $G_{\mathcal{C}_\ell}$~\cite{bayram2025constructing}. Each protected agent initializes its gossip states $\pi_r \in \mathbb{R}^{pT}$ at iteration $s$ as:
\begin{align}\label{eqn:init_z}
\pi_r[0] := P_r^u(\beta^s), \quad r\in\mathcal{C}_\ell,
\end{align}
Agents then exchange information over $M$ gossip iterations and update their states $\pi_r[m]$ at every iteration $m$ to compute the uniform average of the initial value. To preserve privacy, we adopt the push-sum protocol of~\cite{gao2018privacy} for distributed averaging in Algorithm~\ref{alg:priv_gossip}, which prevents revealing the initial states $\{\pi_r[0]\}_{r \in \mathcal{C}_\ell}$. Consequently, the private samples which are encoded in the local projections $P^u_r(\beta^s)$ are protected.

To introduce privacy, during the first $K$ iterations of the gossip process, each agent $r$ generates two independent nonnegative stochastic vectors 
$p_r^{i}$ for $i \in \{1,2\}$ such that 
\begin{align}\label{eqn:stochastic_generate} \
p^i_{rj} =
\begin{cases}
\geq \kappa, & j \in {N}_{r,\mathcal{C}_\ell} \cup \{ r \},\\ 
0, & j \notin {N}_{r,\mathcal{C}_\ell}  \cup \{ r \},
\end{cases}
\quad\!\!\sum_{j=1}^{|V|} p^i_{rj}\!=\!1 ,
\end{align}
$\kappa>0$ is a lower bound on the entries corresponding to neighbor agents. After the masking horizon $K$, each agent switches to a single stochastic vector $p_r$ satisfying \eqref{eqn:stochastic_generate}. These random weights at the first $K$ iterations mask the exchanged auxiliary variables, so that neighboring agents only observe scaled versions of the transmitted states with unknown coefficients. Each agent then updates its local state with the information collected from its neighbors (Lines $14$ and $15$ in Algorithm~\ref{alg:priv_gossip}). At the final gossip round $M$, each agent computes the {\em elementwise} estimate:
$\hat{\pi}_{r,s}=\pi_r[M]/\omega_r[M]$.

\begin{algorithm}
\caption{\textsc{Private Push-Sum}\cite{gao2018privacy}}
\label{alg:priv_gossip}
\begin{algorithmic}[1]
\STATE \textbf{Input:} $\theta_r^s$, $u$
    \STATE \textbf{For all} agent $r \in \mathcal{C}_\ell$ \textbf{(parallel)}
            \STATE $\pi_{r}\!\gets \theta_r^s $ \quad \mbox{ where } $ \theta^s_r= P^{u}_{r}(\beta^{s})$ \textit{from Algorithm~\ref{alg:main} and Algorithm~\ref{alg:project}}

        \STATE $\omega_{r} \gets \mathbf{1}_{pT}$ \quad \mbox{where } $\mathbf{1}_{pT}$ \mbox{is vector of all ones}
\FOR{$m=0$ to $M-1$ \textbf{in parallel }  $\forall r \in \mathcal{C}_\ell$}
    \IF{$m < K$}
    \STATE Generate $p_{r}^{1},\, p_{r}^{2}
    \in \Delta^{|V|}$ (see~\eqref{eqn:stochastic_generate}),
    \ELSE 
    \STATE  Generate $p_{r}\!
    \in\!\Delta^{|V|}$ with $\kappa>0$ (see~\eqref{eqn:stochastic_generate})
    \STATE $p_{r}^{1} \gets p_{r}$ \mbox{ and } $p_{r}^{2} \gets p_{r}$
    \ENDIF
   \STATE Send  $p^1_{rj} \, \pi_r$  and   $p^2_{rj} \, \omega_{r}$ to all $j \in {N}_{r,\mathcal{C}_\ell} $
   \STATE  Collect $p^{1}_{jr}\, \pi_{j}$ and $p^{2}_{jr}\, \omega_{j}$ from all $j \in {N}_{r,\mathcal{C}_\ell}$.
    \STATE $\pi_{r} 
    \gets \sum_{j \in {N}_{r,\mathcal{C}_\ell} \cup \{r\}}  p^{1}_{j r}\, \pi_{j} $

    \STATE $\omega_r \gets \sum_{j \in {N}_{r,\mathcal{C}_\ell} \cup \{r\}}  
     p^{2}_{j r}\, \omega_{j} $

\ENDFOR
\STATE Each node $r$ computes local estimation: $
\hat{\pi}_{r,s} \gets \frac{\pi_{r}}{w_{r}}
$
\STATE \textbf{Return:} $\hat{\pi}_{r,s}$
\end{algorithmic}
\end{algorithm}

{Then, we have the following convergence theorem for Algorithm~\ref{alg:priv_gossip}, in which we extend the asymptotic convergence result of~\cite[Thm.~1]{gao2018privacy} by establishing an explicit geometric convergence rate for the non-masking phase $m > K$:

\begin{theorem}\label{thm:gossip_convergence} 
Suppose that the subgraph $G_{\mathcal{C}_\ell}$ induced by a nonempty $\mathcal{C}_\ell$ is connected. Let $\hat{\pi}_{r,s}[m]=\pi_r[m]/\omega_r[m]$ denote the estimate at agent $r$ produced by Algorithm~\ref{alg:priv_gossip} with $\kappa > 0$ at iteration $m$. Let $\bar{\pi}=\frac{1}{|\mathcal{C}_\ell|}\sum_{r\in \mathcal{C}_\ell} \pi_r[0].$ Then, $\hat{\pi}_{r,s}[m]$ converges to $\bar{\pi}$ at a geometric rate. 
\end{theorem}
\begin{corollary}\label{cor:gossip}
 If each agent $r \in \mathcal{C}_\ell$ initializes $\pi_r[0] := P^u_r(\beta^s)$ for a fixed iteration $\beta^s$, then the estimate $\hat{\pi}_{r,s}[m]$ converges geometrically to $\frac{1}{|\mathcal{C}_\ell|}\sum_{r\in \mathcal{C}_\ell} 
P^u_r(\beta^s)$.
\end{corollary}
}
See Section~\ref{sec:proof} for a proof of Theorem~\ref{thm:gossip_convergence}. This result establishes the convergence of the inner loop. The convergence rate in Theorem~\ref{thm:gossip_convergence} depends on the graph topology. A denser graph yields a faster rate, which is discussed in Section~\ref{sec:numerical}.

\subsubsection{Iterative Projection}\label{sss:proj}

We study the convergence of the middle loop. To this end, we define the simultaneous projection operator $\tilde{T}^{u} : \cV \to \cV$, which is the one step of the middle loop and also $s$-fold composition of $\tilde{T}^u$

\begin{equation}\label{eqn:iterative_projection}
\tilde{T}^{u}_s (\cdot) := \underbrace{\tilde{T}^{u} \circ \cdots\!\circ\!\tilde{T}^{u}}_{s \text{ times}}(\cdot) \text{ where }\tilde{T}^{u}(\cdot)\!:=\!\frac{1}{|\mathcal{C}_\ell|}\!\sum_{r \in \mathcal{C}_\ell}\!P^u_r(\cdot)      
\end{equation}

From Algorithm~\ref{alg:priv_gossip} and Corollary~\ref{cor:gossip}, it follows that $\beta^s\!=\!\tilde{T}^{u}_s(\beta^0)$. Then, we have the following Theorem for the geometric convergence of $\tilde{T}^u_s$:

\begin{theorem}\label{thm:decentralized_sim_convergence}
Suppose (A.1-A.3) hold and the protected agent set $\mathcal{C}_\ell$ is a nonempty. Let $u \in \bigcap_{r \in \mathcal{C}_\ell} U^r $. Then, the sequence of projection operators $\tilde{T}^{u}_s$ induced by~\eqref{eqn:iterative_projection} converges geometrically in operator norm to $P^u_{\mathcal{C}^k_\ell}$.
\end{theorem}

Now, we can extend this result to the released update $\xi$ via the following Corollary:

\begin{corollary}\label{cor:intersection}
    If $\beta^0$ in~\eqref{eqn:iterative_projection} is set to $\xi$ in~\eqref{eqn:dp_gradient_update}, then $\xi$ is projected onto $\operatorname{ker}(u,\mathcal{C}_\ell)$ as $s \to \infty$ at a geometric rate.
\end{corollary}
 See Section~\ref{sec:proof} for a proof of Theorem~\ref{thm:decentralized_sim_convergence}. Since both inner and middle loops converge geometrically from Corollaries~\ref{cor:gossip} and~\ref{cor:intersection}, their nested execution  converges at a geometric rate for a sufficiently large  $M$, which is discussed in Section~\ref{sec:numerical}.

\subsection{Privacy Guarantees for Protected and Learner Agent}\label{subsec:privacy}

In this section, we establish privacy guarantees for the remaining agent roles. The privacy guarantees for the protected agents and the learner agent are established in Propositions~\ref{prop:protected_privacy} and~\ref{prop:privacy}, respectively. Then, we discuss an edge case and our proposed solution to it.

\paragraph{Privacy of the Protected Agents:} We adopt the honest-but-curious adversary model from~\cite[Definition 1]{gao2018privacy}, in which each agent follows the protocol correctly but attempts to infer the initial states of other agents from received messages, without collusion. Then, the privacy of the initial states $\pi_r[0]$ in Algorithm~\ref{alg:priv_gossip}, depends on the topology of the graph.

\begin{proposition}\label{prop:protected_privacy}
Suppose ${G}_{\mathcal{C}_\ell^k}$ is connected and all agents are  honest-but-curious for all $k$. Then, for any protected agent $r \in \mathcal{C}^k_\ell$ such that $|{N}_{r,\mathcal{C}^k_\ell} |\geq 2$, no other agent running Algorithm~\ref{alg:priv_gossip} can uniquely determine the initial state of agent $r$ and cannot determine a finite range for $\pi_{r}[0]$.
\end{proposition}
\begin{proof} We first construct a directed graph $\vec{G}_{\mathcal{C}_\ell}$ from $G_{\mathcal{C}_\ell}$ by replacing each undirected edge with two directed edges. Under this construction, $\vec{G}_{\mathcal{C}_\ell}$ is strongly connected. All nodes follow Algorithm~\ref{alg:priv_gossip}. Then, from~\cite[Theorem 2]{gao2018privacy}, the privacy of a node $r$ is preserved if the cardinality of the union of its in-neighbors and out-neighbors is greater than one, i.e. $|{N}^{in}_{r,\mathcal{C}_\ell} \cup {N}^{out}_{r,\mathcal{C}_\ell}|\geq 2$. By construction, the in-neighbor and out-neighbor sets of every agent $r$ coincide in $\vec{{G}_{\mathcal{C}_\ell}}$, from which $|{N}_{r,\mathcal{C}_\ell} |\geq 2$ follows.
\end{proof}

\paragraph{Privacy of the Learner Agent:} To formalize privacy of the learner agent, we adopt the notion of indistinguishability because it depends on whether an adversary can uniquely recover its sample $x^\ell$ from the shared sequence of controls in the network. To this end, for a sequence of control $\Gamma$ and a label $y^i$, we define the \emph{indistinguishability set} of $x^i$ over $\Gamma$: 
\begin{align}\label{eqn:defn_I_set}
    \mathcal{I}_\Gamma(x^i) := 
    \{x \in \mathbb{R}^{\bar{n}} \mid R(\varphi(u,x)) = y^i,\ 
    \forall u \in \Gamma\}.
\end{align}
\begin{definition} For a given control sequence $\Gamma$, the learner agent {preserves privacy over $\Gamma$} if 
$$\operatorname{card}(\mathcal{I}_\Gamma(x^\ell)) > 1.$$
\end{definition}

In words, if the control sequence produces the same output $y^i$ under $R(\varphi(u,\cdot))$ for more than one initial point, the learner agent preserves privacy over $\Gamma$. To this end, we define the sequence of controls $\Gamma^q_{\mathcal{C}_\ell}$ that is shared with the network during the outer iteration in which the learner memorizes the sample of teacher agent $q$ with a nonempty set $\mathcal{C}_\ell$. The following proposition shows that TwF inherently provides indistinguishability under $\Gamma^q_{\mathcal{C}_\ell}$:

\begin{proposition}\label{prop:privacy}
If there exists $m \in \mathcal{C}_\ell$ with $m \neq \ell$ and $y^m = y^\ell$,
then $\mathrm{card}\big(I_{\Gamma^q_{\mathcal{C}_\ell}}(x^\ell)\big) > 1$.
\end{proposition}
\begin{proof}
Since $m \in \mathcal{C}^k_\ell$ for all $k$ during the teacher phase of agent $q$,
TwF restricts each update $\delta u^k$ to $T_u(\bigcap_{r \in \mathcal{C}^k_\ell} U^r)$.
Hence the sequence $\Gamma^q_{\mathcal{C}_\ell}$ remains in the submanifold
$\bigcap_{r \in \mathcal{C}^k_\ell} U^r$. Since $y^m = y^\ell$, it holds that
$R(\varphi(u,x^m)) = y^\ell$ for all $u \in \Gamma^q_{\mathcal{C}_\ell}$. Hence the
indistinguishability set of $x^\ell$ over $\Gamma^q_{\mathcal{C}_\ell}$ contains at
least the two elements $\{x^m, x^\ell\}$, where $x^m \neq x^\ell$.
\end{proof}

Both Proposition~\ref{prop:privacy} and Proposition~\ref{prop:protected_privacy} require $\mathcal{C}_\ell$ to be large enough for privacy of the learner. At initialization, when $\mathcal{C}_\ell=\{\ell\}$, neither condition is guaranteed. To resolve this, before the main loop via Algorithm~\ref{alg:phase1}, the learner introduces a fictitious sample for a fictitious agent $\ell'$:
\begin{align}
    x^{\ell'} \quad \text{s.t.} \quad 
    0 < \|x^{\ell'} - x^\ell\|_2 \leq \kappa, \quad 
    y^{\ell'} = y^\ell,
\end{align}
and memorizes it. Thus, the learner sets $\mathcal{C}_\ell \leftarrow \{\ell, \ell'\}$ before Algorithm~\ref{alg:main} begins. The fictitious agent $\ell'$ participates in the gossip process as an additional node connected to $\ell$ and its neighbors. 
This construction satisfies $|N_{\ell, \mathcal{C}_\ell}| \geq 2$ for Proposition~\ref{prop:protected_privacy} and the label 
co-occurrence condition for Proposition~\ref{prop:privacy} so that the privacy of the learner agent is guaranteed. Once $\mathcal{C}_\ell$ expands sufficiently, all fictitious samples are removed. Removing a fictitious sample enlarges $\ker(u, \mathcal{C}\ell)$ and thus relaxes the constraint on subsequent updates. Since TwF subsequently forgets any sample removed from $\mathcal{C}\ell$, this removal does not affect model performance.


%

\begin{remark}
Any other agent $r \in \mathcal{C}_\ell^k$ with $|N_{r, \mathcal{C}_\ell^k}| < 2$ may apply the same construction by Algorithm~\ref{alg:phase1} to satisfy Proposition~\ref{prop:protected_privacy} for any given connected graph.    
\end{remark}

\begin{algorithm}
\caption{Fictitious Sample For Agent $r$}
\label{alg:phase1}
\begin{algorithmic}[1]
\STATE \textbf{Select}  $ x^{r'} \quad  \mbox{s.t.} \quad  0 < \| x^{r'} - x^{r} \|_2 \leq \kappa  \mbox{ and } y^{r'} = y^r$
\FOR{$j \in \{r, r'\}$}
    \REPEAT
        \STATE $u \gets u - \alpha \,
        P_{\mathcal{C}}^u\!\big(
        \nabla {J}_j( u + \gamma^*_{j,u} )
        \big)$
    \UNTIL{convergence for sample $j$}
    \STATE $\mathcal{C} \gets 
    \mathcal{C} \cup \{ j \}$
\ENDFOR
\STATE \textbf{Return:} $u$, $\mathcal{C}$
\end{algorithmic}
\end{algorithm}

\section{Proofs}\label{sec:proof}
\begin{proof}[Proof of Proposition~\ref{prop:dp_teacher}] We fix the control $u$ at the current iteration. Let $D=(x^q,y^q)$ and let $D'=(x^{q'},y^{q'})$ be an adjacent dataset of $D$. The released update $\xi$ in~\eqref{eqn:mechanism} is computed in two sequential steps, each depending on the private 
sample $(x^q, y^q)$. The clipped gradient $g_{q,\bar{u}}$ is evaluated at
$\bar u = u + \gamma^{*}_{q,u}$, and $\gamma^{*}_{q,u}$ is a deterministic
function of $D$ that is never released. Hence $D \mapsto P_\Lambda(g_{q,\bar{u}})$
is a deterministic query of the dataset, and $\mathcal{M}_\Lambda$ is the
Gaussian mechanism applied to it.

Let $\{\zeta_l\}_{l=1}^d$ be the orthonormal basis for given $\Lambda$. Then, Gaussian process coefficients from~\eqref{eqn:defn_noise_process} satisfy $\langle \mathcal{W}_\Lambda, \zeta_l \rangle = \chi_l $ where $\chi_l \sim \mathcal{N}(0,\sigma^2 B^2 )$, 
and are mutually independent. Then, we have the released vector~\eqref{eqn:mechanism}:

\begin{align}
  \bigl(\langle \mathcal{M}_{\Lambda}(D),
  \zeta_l\rangle\bigr)_{l=1}^d
  = h_D +
  \bigl(\langle\mathcal{W}_\Lambda,
  \zeta_l\rangle\bigr)_{l=1}^d \in\mathbb{R}^{d}
\end{align}
where $ h_D := \bigl( \langle g_{q,\bar{u}},
  \zeta_l\rangle\bigr)_{l=1}^d
  \in \mathbb{R}^{d}$.  Let $f := g_{q,\bar{u}} - g_{q',\bar{u}'}$. Since the basis $\{\zeta_l\}_{l=1}^d$ is orthonormal, Parseval's identity gives:
\begin{align}\label{eqn:sensitivity}
    \sup_{D \sim D'} \|h_D - h_{D'}\|_2^2
    = \sup_{D \sim D'} \sum_{l=1}^d |\langle f, \zeta_l \rangle|^2 \leq 4B^2,
\end{align}
{where the last bound follows from the clipping~\eqref{eqn:gradient}, which gives
$\|g_{q,\bar u}\| \le B$ for every admissible dataset.} By~\cite[Thm.~A.1]{dwork2014algorithmic}, a Gaussian mechanism with diagonal covariance $\Sigma = \sigma^2 B^2 I$ satisfies $(\varepsilon, \delta)$-DP 
if
\begin{align}\label{eqn:dwork}
    \sup_{D \sim D'} \|\Sigma^{-1/2}(h_D - h_{D'})\|_2 
    \leq \frac{\varepsilon}{\sqrt{2\log(1.25/\delta)}}.
\end{align}
From~\eqref{eqn:sensitivity}, the condition~\eqref{eqn:dwork} holds if $\sigma$ satisfies~\eqref{eqn:prop_bound}, from which the result follows. Since $D$ and $D'$ range over all admissible single-sample datasets, the
guarantee is local $(\varepsilon,\delta)$-DP on $\Lambda$.
\end{proof}

\begin{proof}[Proof of Theorem~\ref{thm:gossip_convergence}] We follow the construction of the directed graph $\vec{G}_{\mathcal{C}_\ell}$ given in the proof of Proposition~\ref{prop:protected_privacy}. Under this construction, the in-neighbor and out-neighbor sets of every agent $r$ coincide and $\vec{G}_{\mathcal{C}_\ell}$ is
strongly connected. Then, from~\cite[Theorem~1]{gao2018privacy}, the private push-sum protocol ensures that the estimate of each participating agent $\hat{\pi}_{r,s}[m]$ converges to the average of the initial states of participating agents $\bar{\pi}$ if each agent follows Algorithm~\ref{alg:priv_gossip} with $\kappa>0$ and finite $K$. It remains to establish the geometric rate.

Without loss of generality, we take $\pi_r \in \mathbb{R}$ and let $a := |\mathcal{C}_\ell|$, the general case follows from component-wise scaling. Then, we collect the states of all protected agents into a state vector
$$
    \pi[m] := [\pi_{r_1}[m], \ldots, \pi_{r_{a}}[m]]
    \in \mathbb{R}^{a},
$$ where $r_1, \ldots, r_{a}$ is any fixed indexing of $\mathcal{C}_\ell$.
For $m \geq K+1$, we define the communication matrix 
${P}[m] \in \mathbb{R}^{a \times a}$ with entries
$[{P}[m]]_{r_i r_j} := p_{r_i r_j}[m]$, which are generated by~\eqref{eqn:stochastic_generate}. ${P}[m]$ is row-stochastic by $\sum_{r_j} p_{r_i r_j}[m] = 1$. Then, the update rules in Lines~$14-15$ of Algorithm~\ref{alg:priv_gossip} follow:
\begin{align}\label{eqn:markov}
    \pi[m+1] = \pi[m] {P}[m]
    \qquad 
    \omega[m+1] = \omega[m] {P}[m]
\end{align}
Let $\hat{\pi}[m] \in \mathbb{R}^{a}$ be the estimate vector at iteration $m$, whose $r$-th entry is $\hat{\pi}_r[m] := \pi_r[m]/\omega_r[m]$.  Let  
$\bar{\pi} := \frac{1}{a}\sum_r \pi_r[0]$, where we have 
$\omega_r[0] = 1$ for all $r$. Then, we define the error vector as follows:
$$
e[m] := \pi[m] - \bar{\pi}\,\omega[m] \in \mathbb{R}^{a}.
$$
Since both $\pi$ and $\omega$ evolve by the same matrix 
in~\eqref{eqn:markov} for $m > K$, the error $e$ also evolves by the same matrix $P[m]$:
$$
e[m+1] = e[m]\,P[m], \quad m > K.
$$
By row-stochasticity and $\omega_r[0] = 1$, mass is conserved: 
$\sum_r \omega_r[m] = a$ and $\sum_r \pi_r[m] = \sum_r \pi_r[0]$ 
for all $m$, which gives $\sum_r e_r[m] = 0$ for all $m$.  Since $\omega_r[0] = 1$ and $p^2_{rr}[m] \geq \kappa > 0$ for all $m \geq 0$, each agent retains a positive fraction of its own mass at every iteration, so 
$\omega_r[m] > 0$ for all $m \geq 0$ and all $r$. Combined with 
mass conservation $\sum_r \omega_r[m] = a$, there exists 
$c_\omega > 0$ such that $\omega_r[m] \geq c_\omega$ for all 
$m > K$ and all $r$.

Since $\omega_r[m] \geq c_\omega > 0$, 
we have the following upper bound for each $r$:
\begin{align}\label{eqn:error_bound}
|\hat{\pi}_r[m] - \bar{\pi}| 
= \frac{|\pi_r[m] - \bar{\pi}\,\omega_r[m]|}{\omega_r[m]}
= \frac{|e_r[m]|}{\omega_r[m]}
\leq \frac{\|e[m]\|_\infty}{c_\omega}.
\end{align}
It therefore suffices to establish geometric decay of $\|e[m]\|_\infty$.

Since $G_{\mathcal{C}_\ell}$ is connected and $p_{r_i r_i}[m] 
\geq \kappa > 0$ for all $r_i$, ${P}[m]$ is irreducible 
with positive diagonal entries for each $m > K$. By~\cite[Lemma 12]{bayram2023vector}, the product of any 
$l \geq \lfloor a /2 \rfloor$  irreducible row-stochastic 
matrices of dimensions $a\times a$ with positive diagonal entries is a scrambling matrix. Let $\upsilon=\lfloor a/2\rfloor$.
Hence ${Q} := \prod_{t=m_0}^{m_0+\upsilon-1} 
{P}[t]$ is scrambling for any $m_0 \geq K+1$, with all 
entries bounded below by $\kappa^{\upsilon} > 0$. Then,  the ergodicity coefficient satisfies $\mu({Q}) \leq 1 - \kappa^{\upsilon}$. We then apply the submultiplicative inequality $\|B_1B_2\|_S \leq \mu(B_1)\|B_2\|_S$ for any two stochastic matrices $B_1$ and $B_2$~\cite{horn2012matrix} over consecutive blocks of length $\upsilon$. This yields the geometric rate for each $r$ when $m>K$.
\end{proof}

\begin{proof}[Proof of Theorem~\ref{thm:decentralized_sim_convergence}] Let ${M} := \bigcap_{r \in \mathcal{C}_\ell^k} T_u U^r$ and let $P_{{M}} : \cV \to {M}$ denote the orthogonal projection onto ${M}$. Under Assumptions~(A.1)-(A.3), \cite[Corollary~1]{bayram2025geometric} ensures that ${M}$ is a nonempty closed subspace of $\cV$. Under the nonsingularity assumption~(A.2), by~\cite[Proposition~2]{bayram2025geometric}, the set $U^r$ 
is a finite-codimension Hilbert submanifold of $\cV$ for each $r \in \mathcal{C}_\ell$. Hence $T_u U^r$ is a closed linear subspace of finite codimension, its orthogonal complement $T_u^\perp U^r$ is finite-dimensional and closed. Since $\mathcal{C}_\ell$ is finite, 
the sum $\sum_{r \in \mathcal{C}_\ell} T_u^\perp U^r$ is a finite sum of closed subspaces and is therefore closed. By~\cite[Theorem~1.2]{reich2021error}, the closedness of $\sum_{r \in \mathcal{C}_\ell^k} T_u^\perp U^r$ implies that the sequence $\{\tilde{T}^u_s\}_{s \geq 0}$ 
in~\eqref{eqn:iterative_projection} converges geometrically  in operator norm to the $P_{{M}}$.\end{proof}

\section{Numerical Studies}\label{sec:numerical}

We consider two experimental settings. For memory preservation 
and convergence of the distributed projection, we use a 2D 
\textbf{circle classification task}, in which initial points $\bar{x}^i \in \mathbb{R}^2$ 
are classified as inside or outside the unit disk, with uplift 
$E:\mathbb{R}^2\to\mathbb{R}^{10}$ and a readout map $R$ 
projecting onto the last two coordinates (2-class one-hot).

For the privacy-performance tradeoff, we use \textbf{MNIST image classification}~\cite{lecun1998mnist}, in which each image 
is flattened to $\mathbb{R}^{784}$, uplifted to $\mathbb{R}^{834}$. We consider a readout map $R$ 
projecting onto the last 10 coordinates (10-class one-hot). 

\smallskip
\noindent\textbf{Neural ODE:} We consider a 100-layer residual neural network for both tasks, which is an Euler discretization of the following controlled ODE over $[0, 1]$ with $\Delta t =0.01$:
\begin{align}
    \dot{x}(t) = \tanh(W(t)x(t) + b(t))
\end{align}
where $\tanh$ acts component-wise and $u(t) = (W(t),b(t)) \in L^2([0,1],\mathbb{R}^{n \times n}  \times \mathbb{R}^{n})$.

\smallskip
\noindent\textbf{Sequential memorization:}
The right panel of Fig.~\ref{fig:numerical} depicts the per-agent loss 
$J_q(u)$ over global control update iterations $k$ for the first $20$ agents over the fully connected network. Each agent's sample is successfully memorized during its teacher phase, i.e  $J_q(u) \leq 0.1$. Then, each agent's loss remains below $\tau_{\mathrm{mem}}$ throughout all subsequent protected phases (shown in dashed lines). This confirms that the 
distributed computation of the projection $P^u_{\mathcal{C}_\ell}$ prevents forgetting.

\begin{figure*}
    \centering
    \begin{minipage}{0.33\linewidth}
        \includegraphics[width=\linewidth]{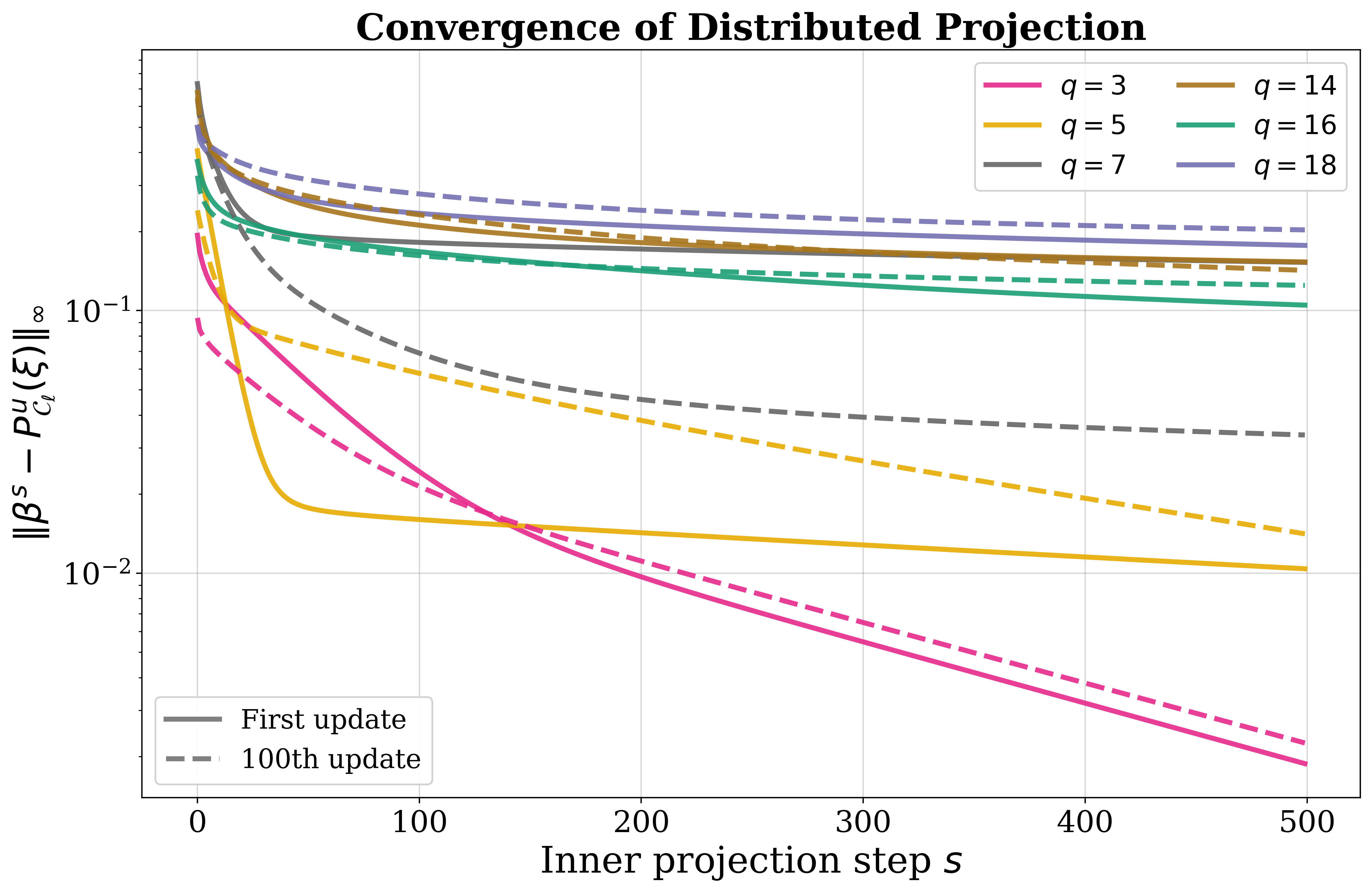}
    \end{minipage}
    \begin{minipage}{0.65\linewidth}
        \includegraphics[width=\linewidth]{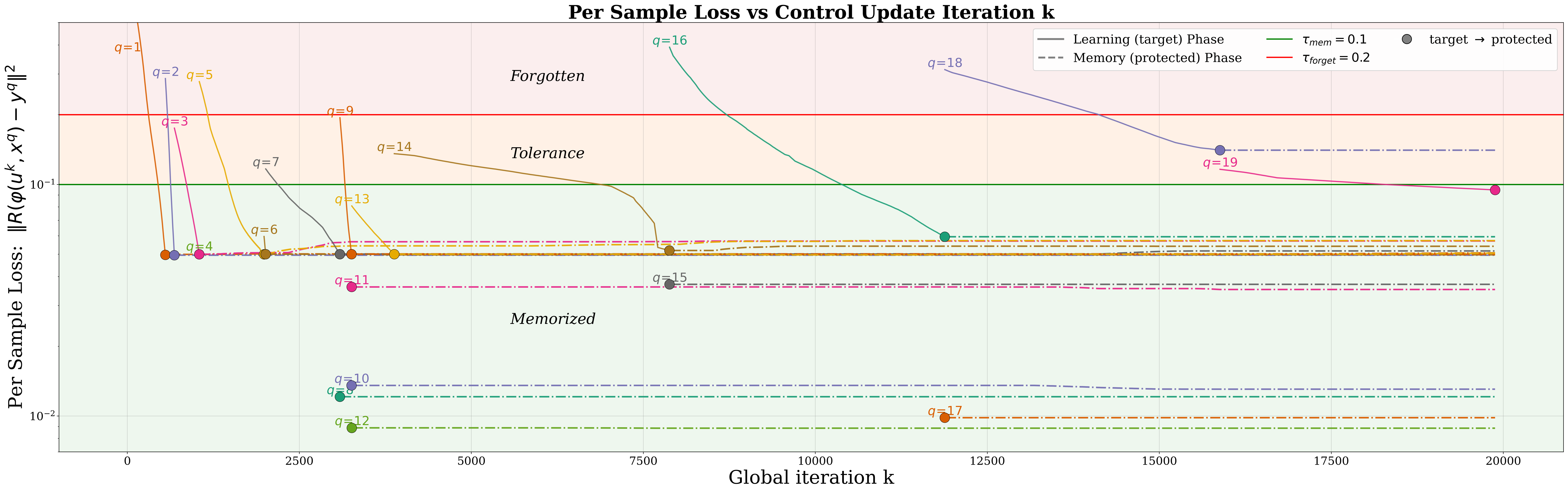}
    \end{minipage}
        \caption{(Left) The error $\|\beta^s - P^u_{\mathcal{C}_\ell}(\xi)\|_\infty$ measures the distance between the current iterate $\beta^s$ and the projection onto the intersection of tangent spaces $P^u_{\mathcal{C}_\ell}(\xi)$ over inner steps $s$. (Right) Evolution of per-agent loss $J_r(u)$ during sequential training for circle classifier. Solid and dashed lines show each agent during its teacher and protected phases, respectively. Circle $(\circ)$ marks $\mathcal{C}_\ell$ expansion iterations. Shaded bands indicate memorized ($J_r <0.1$), tolerance, and forgotten ($J_r > 0.2$) regions. 
The hyperparameters are: $\alpha = 0.01, \rho=0.1, \lambda = 0.2, \varepsilon = 3, \kappa = 0.01, S = 500, M = 500, K = 10$.  } 
\label{fig:numerical}
\end{figure*}

\smallskip
\noindent\textbf{Convergence of private gossip:} 
Fig.~\ref{fig:gossip} depicts the error $\|\hat{\pi}_r[m] - \frac{1}{|\mathcal{C}_\ell|}\sum_{r \in \mathcal{C}_\ell}\pi_r[0]\|_\infty$ as a function of gossip iteration $m$ for fully connected (FC) and ring topologies. In both cases, the error decreases at a geometric rate, which verifies Theorem~\ref{thm:gossip_convergence}. However, the ring network converges slower than the FC network. Privacy is preserved in both cases because each agent has at least two neighbors. As expected, convergence becomes faster for smaller $|\mathcal{C}_\ell|$. Since weights $p_r^1$ and $p^2_r$ are randomly resampled for privacy, the convergence is delayed by $K$ steps.

\smallskip
\noindent\textbf{Convergence of distributed projection:} 
The left panel of Fig.~2 depicts $\|\beta^s - P^u_{\mathcal{C}_\ell}(\xi)\|_\infty$ as a function of the middle loop iteration $s$, for the first and the 100th control updates for various teacher agents $q$. In both cases, the error decreases at a geometric rate, which confirms Theorem~\ref{thm:decentralized_sim_convergence}. The convergence slows down as $|\mathcal{C}_\ell|$ grows and the intersection of tangent spaces becomes more constrained.

\begin{figure}
    \centering
    \includegraphics[width=0.9\linewidth]{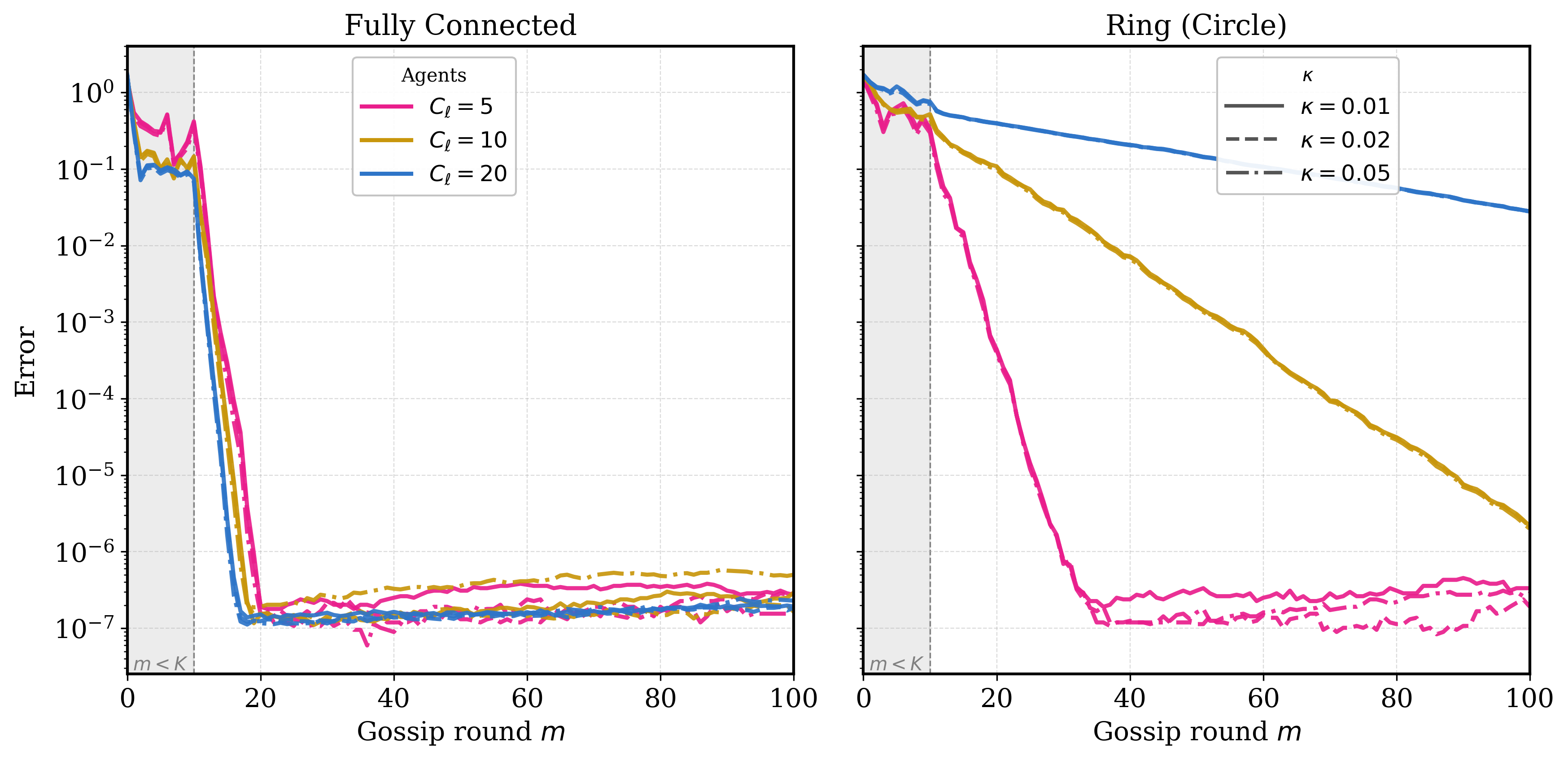}
    \caption{The error $\|\hat{\pi}_r[m] - \frac{1}{|\mathcal{C}_\ell|}\sum_{r \in \mathcal{C}_\ell}\pi_r[0]\|_\infty$ measures the distance between the current estimate $\hat{\pi}_r[m]$ and the uniform average of initial states as a function of gossip round $m$. The shaded region $m < K$ marks privacy phase. (Left) Fully Connected network. (Right) Ring network.}
    \label{fig:gossip}
\end{figure}

\smallskip
\noindent\textbf{Privacy-performance trade-off:}
We evaluate the proposed methods under different privacy budgets $\varepsilon \in \{1,2,3\}$ and non-private update $\varepsilon\!=\!\infty$. The teacher phase of an agent takes $K_q \approx 10^3$ control updates in these experiments. Thus, the budget over a
full teacher phase follows from the composition bound in Section~\ref{subsec:learning_new}. Figure~\ref{fig:acc_error} depicts the test accuracy on MNIST dataset and circle classification task as a function of privacy budget for $10$ different runs. For Robust DP (R-DP), we use the proposed release vector in~\eqref{eqn:dp_gradient_update} with $\rho=0.1$. For DP, we clip the gradient $\nabla_u J_q(u^k)$ and add the Gaussian noise process. As $\varepsilon$ increases, the privacy reduces. Robust DP mitigates the performance loss across all privacy budgets, achieving better accuracy than DP while preserving the same privacy guarantee in both MNIST and circle classification.

\begin{figure}
    \centering
    \includegraphics[width=0.9\linewidth]{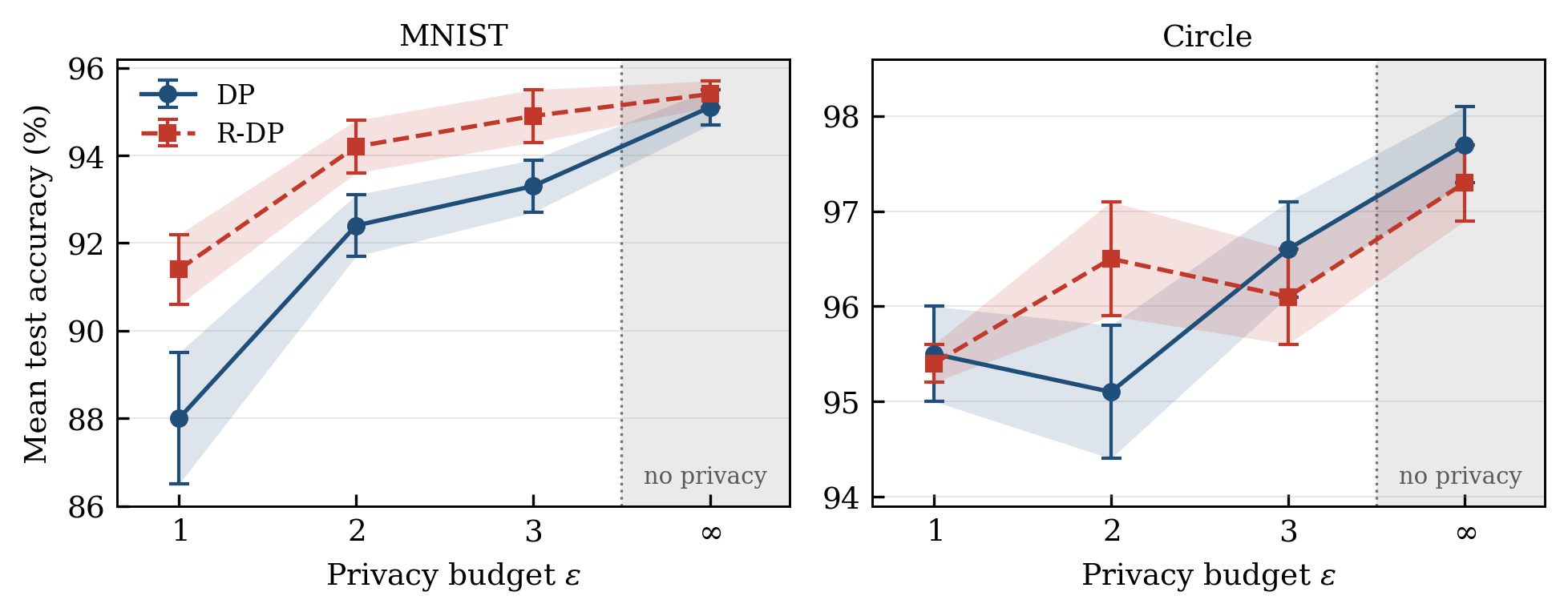}
     \caption{Mean test accuracy as a function of privacy budget $\varepsilon$ for DP and R-DP on MNIST (left) and Circle (right). Markers show the mean over runs and shaded bands show one standard deviation. The shaded region at $\varepsilon=\infty$ is the non-private baseline.}
    \label{fig:acc_error}
\end{figure}

\section{Conclusion and Future Work} 
In this work, we have introduced a privacy-preserving distributed training algorithm that combines gossip consensus with Tuning without Forgetting. We have provided theoretical privacy guarantees for all three agent roles. We have shown that the robust formulation recovers the performance loss due to noise injection.  We have proved geometric convergence of the distributed TwF projection scheme.

Future work will focus on modeling the injected noise in a distributional framework rather than under a worst-case perturbation assumption. We will also investigate how the geometric convergence rate depends on the underlying graph topology.

\bibliography{learning}
\end{document}